\documentclass[letterpaper, 10 pt, conference]{ieeeconf}  

\IEEEoverridecommandlockouts                              

\usepackage{times}  
\usepackage{helvet}  
\usepackage{courier}  
\usepackage{url}  
\usepackage{graphicx}  
\usepackage{times}
\usepackage{helvet}
\usepackage{courier}

\usepackage{multicol}
\usepackage{algorithm}
\usepackage[noend]{algpseudocode}

\usepackage{amsmath}
\usepackage{amssymb}
\usepackage{graphicx}
\usepackage{mathtools}
\usepackage{tikz}
\usepackage{times}
\usepackage{bm}
\usepackage{subfigure}
\usepackage{multirow}
\usepackage{graphics}
\usepackage{booktabs}
\usetikzlibrary{calc,shapes}

\usepackage{amsthm}

\newtheorem{theorem}{Theorem}

\title{\LARGE \bf
Cooperative Schedule-Driven Intersection Control with Connected and Autonomous Vehicles }

\author{Hsu-Chieh Hu$^{1}$, Stephen F. Smith$^{2}$, Rick Goldstein$^{2}$
\thanks{*This research was funded in part by the University Transportation Center on Technologies for Safe and Efficient Transportation at Carnegie Mellon University and the CMU Robotics Institute.}
\thanks{$^{1}$Hsu-Chieh Hu is with Department of Electrical and Computer Engineering, Carnegie Mellon University, Pittsburgh, Pennsylvania, USA. {\tt\small hsuchieh@andrew.cmu.edu}}%
\thanks{$^{2}$Stephen F. Smith and Rick Goldstein are with the Robotics Institute, School of Computer
Science, Carnegie Mellon University, Pittsburgh, Pennsylvania, USA. 
        {\tt\small\{sfs, rgoldste\}@cs.cmu.edu}}%
}

\begin{document}

\maketitle
\thispagestyle{empty}
\pagestyle{empty}

\begin{abstract}  


Recent work in decentralized, schedule-driven traffic control has demonstrated the ability to improve the efficiency of traffic flow in complex urban road networks. In this approach, a scheduling agent is associated with each intersection. Each agent senses the traffic approaching its intersection and in real-time constructs a schedule that minimizes the cumulative wait time of vehicles approaching the intersection over the current look-ahead horizon. In this paper, we propose a cooperative algorithm that utilizes both connected and autonomous vehicles (CAV) and schedule-driven traffic control to create better traffic flow in the city. The algorithm enables an intersection scheduling agent to adjust the arrival time of an approaching platoon through use of wireless communication to control the velocity of vehicles. The sequence of approaching platoons is thus shifted toward a new shape that has smaller cumulative delay. We demonstrate how this algorithm outperforms the original approach in a real-time traffic signal control problem.

\end{abstract}

\section{Introduction}

Traffic congestion has been becoming an increasingly critical problem and is getting worse due to population shifts to urban areas and the high usage rate of vehicles. Traffic jams are now common on urban surface streets, where It is commonly recognized a better optimization of traffic signals could lead to substantial reduction of traffic congestion. A recent development in decentralized online planning to traffic signal control problem has achieved significant improvements to urban traffic flows through real-time, distributed generation of long-horizon, signal timing plans. \cite{xie2012schedule,smith2013smart} The key idea behind this {\em schedule-driven} approach is to formulate the intersection scheduling problem as a single machine scheduling problem, where input jobs are represented as sequences of spatially proximate vehicle clusters (approaching platoons, queues). This aggregate representation enables efficient generation of long horizon timing plans that incorporate multi-hop traffic flow information, and thus network-wide coordination is achieved through exchange of schedule information among neighboring intersections. 

Within the schedule-driven traffic control approach, a generated signal timing plan is executed in rolling horizon fashion, and is recomputed every second or so to account for uncertain traffic conditions. Each time a timing plan is regenerated, the instantaneous state of approaching cluster sequences is used as a predictive model, to provide a tractable estimation of current demand that preserves the non-uniform nature of real-time traffic flows.
This information, which includes the vehicles that comprise each cluster and their respective arrival times at the intersection, not only provides a basis for generating signal timing plans that optimize the actual traffic on the road, but also paves a way to further optimize flows through collaboration between infrastructure (the intersection scheduling agents) and vehicles. 

The recent development of connected and autonomous vehicle (CAV) technologies (such as vehicle-to-infrastructure (V2I) communication and autonomous vehicle control), provides opportunities to improve the driver's behavior, and, in the longer term, to even control the movement of vehicles, for purposes of improving the efficiency and safety of transportation. The emergence of autonomous vehicle control enables rapid agile response to the traffic information sent by the traffic signal system, which could not necessarily be expected for human drivers. Hence, the availability of such technologies provides new flexibility to the design of new traffic control systems. On the other hand, the rich traffic information that will be made available by V2I communication technologies can enable the development of better traffic control strategies for optimizing a given performance objective, e.g, delay or capacity (throughput).

In this paper, we propose an algorithm that achieves improved traffic flow efficiency through collaboration between schedule-driven traffic control and CAVs. The proposed algorithm takes advantage of the detailed real-time timing information provided by schedule-driven traffic control and the availability of autonomous control (or a driver advisory system) through CAVs to establish a framework for interaction. At the beginning of each planning cycle, the intersection scheduling agent first generates a signal timing plan based on the detected traffic flow as before. Then, based on this timing plan, the agent determines a set of velocities for all approaching vehicles (essentially speeding up and slowing down selected approaching vehicles) that enables generation of a better signal timing plan.  After receiving its suggested speed, each approaching vehicle controls its speed accordingly. The cluster representation is adjusted through iterative application of two operations, each of which involve merging and splitting of existing clusters: (1) Re-sequencing vehicles belonging to different traffic flows (i.e., belonging to different signal phases) by schedule-driven traffic control, and (2) Adjusting the speed of approaching vehicles and intersections such that the gaps between clusters competing for green time are enlarged. It is shown that, in scenarios with heavy congestion, the cooperative method outperforms schedule driven control, resulting in an additional delay reduction of $19\%$. Furthermore, this algorithm is still feasible in the case that the penetration rate of CAV is low.

The remainder of the paper is organized as follows. We first introduce related work and schedule-driven traffic control. Next, the algorithm necessary to achieve cooperation and its theoretical analysis are discussed. Then, an empirical analysis of the composite approach  is presented. Finally, some conclusions are drawn.

\section{Related Work}


Traditionally, there are three general approaches to control traffic signal: a) fixed timing; b) actuated; and c) adaptive. The earliest implementations are based on a fixed timing method optimized using historical traffic data offline. The later advancements have used actuated or adaptive signals. Then, if all cars are equipped with wireless communication technologies, e.g., Dedicated Short Range Communications (DSRC), to communicate with a centralized infrastructure, we can optimize the traffic flow by ordering the phases of traffic signal more efficiently. In \cite{priemer2009decentralized,goodall2013traffic} information from equipped vehicles is used to determine demands and optimize the cycle length and green splits of a traffic signal once every cycle. In \cite{he2012pamscod} the presence of platoons is detected using V2I communication and a mixed integer non-linear program is solved to produce the optimal phasing sequence. However, this approach does not scale well for generating long horizon plans and does not deal with uncertainty of traffic states.

Traffic flow can also be optimized by searching for the optimal discharge sequence for each individual vehicle. One possible approach is request-based framework where each vehicle reserves a space through sending request in advance. For example, \cite{dresner2008multiagent} proposed a novel intersection control method called Autonomous Intersection Management (AIM), and in particular described a First Come, First Served (FCFS) policy to direct autonomous vehicles when to pass through the intersection. They showed that by leveraging the capacity of computerized driving systems FCFS significantly outperforms traditional traffic signals. Another similar approaches is \cite{tachet2016revisiting} in which the system is slot-based and able to double capacity. Although those approaches are promising, they assume perfect penetration of connected vehicles and may not be realized in near future. We need a more practical approach to bridge between traditional methods (i.e., traffic signals) and request-based methods.

Taking advantage of autonomously controlled vehicles, and the use of information from connected vehicles for intersection control has been investigated in several researches. The trajectory of fully autonomous vehicles can be manipulated to optimize
an objective function \cite{li2006cooperative,zohdy2012game,lee2012development,liang2018signal}. Those approaches can achieve either better safety or efficiency through interaction between intersections and vehicles. \cite{au2010motion} propose an extension of AIM to enable vehicles to apply motion planning for optimizing speed.

\section{Schedule-Driven Traffic Control}
As indicated above, the key to the single machine scheduling problem formulation of the schedule-driven approach of \cite{xie2012schedule} is an aggregate representation of traffic flows as sequences of clusters $c$ over the planning (or prediction) horizon. Each \textit{cluster} $c$ is defined as $(|c|, arr, dep)$, where  $|c|$, $arr$ and $dep$ are number of vehicles, arrival time and departure time respectively. Vehicles entering an intersection are clustered together if they are traveling within a pre-specified interval of one another. The clusters become the jobs that must be sequenced through the intersection (the single machine). Once a vehicle moves through the intersection, it is sensed and grouped into a new cluster by the downstream intersection.The sequences of clusters provide short-term variability of traffic flows at each intersection and preserve the non-uniform nature of real-time flows. Specifically, the \textit{road cluster sequence} $C_{R,m}$ is a sequence of $(|c|, arr, dep)$ triples reflecting each approaching or queued vehicle on entry road segment $m$ and ordered by increasing $arr$. Since it is possible for more than one entry road to share the intersection in a given \textit{phase} (a phase is a compatible traffic movement pattern, e.g., East-West traffic flow), the \textit{input cluster sequence} $C$ can be obtained through combining the road cluster sequences $C_{R,m}$ that can proceed concurrently through the intersection. The travel time on entry road $m$ defines a finite horizon ($H_m$), and the prediction horizon $H$ is the maximum over all roads.

Every time the cluster sequences along each approaching road segment are determined, each cluster is viewed as a non-divisible job and a forward-recursion dynamic programming search is executed in a rolling horizon fashion to continually generate a phase schedule that minimizes the cumulative delay of all clusters. The frequency of invoking the intersection scheduler is once a second, to reduce the uncertainty associated with clusters and queues. The process constructs an optimal sequence of clusters that maintains the ordering of clusters along each road segment, and each time a phase change is implied by the sequence, then a delay corresponding to the intersection's yellow/all-red changeover time constraints is inserted based on Algorithm \ref{pd}.  If the resulting schedule is found to violate the maximum green time constraints for any phase (introduced to ensure fairness), then the first offending cluster in the schedule is split, and the problem is re-solved.

Formally, the resulting \textit{control flow} can be represented as a tuple $(S, C_{CF})$ shown in Figure~\ref{cf}, where $S$ is a sequence of phase indices, i.e., $(s_1, \cdots, s_{|S|})$, $C_{CF}$ contains the sequence of clusters $(c_1, \cdots, c_{|S|})$ and the corresponding starting time after being scheduled. More precisely, the delay that each cluster contributes to the cumulative delay $\sum_{k = 1}^{|S|} d(c_{k})$ is defined as 
\begin{equation}
d(c_{k}) = |c_{k}| \cdot (ast - arr(c_{k})),
\end{equation} 
where $ast$ is the actual start time that the vehicle is allowed to pass through, which is determined by Algorithm \ref{pd}. Note that $ast$ is determined by the $arr$ and permitted start time ($pst$) described in Algorithm \ref{pd}.  For a partial schedule $S_k$, the corresponding state variables are defined as a tuple, $(s, pd, t, d)$, where $s$ is phase index and $pd$ is duration of the last phase, $t$ is the finish time of the $k$th cluster and $d$ is the accumulative delay for all $k$ clusters. The state variable of $S_k$ can be updated from $S_{k-1}$. Algorithm \ref{pd} is based on a greedy realization of planned signal sequence, where $MinSwitch(s, i)$ returns the minimum time required for switching from phase $s$ to $i$ and $slt_i$ is the \textit{start-up lost time} for clearing the queue in the phase $i$. We can use $t$ and $MinSwitch(s, i)$ to derive $pst$. The optimal sequence (schedule) $C_{CF}^*$ is the one that incurs minimal delay for all vehicles.

 \begin{figure}[!htbp]
\centering
\includegraphics[scale = 0.35]{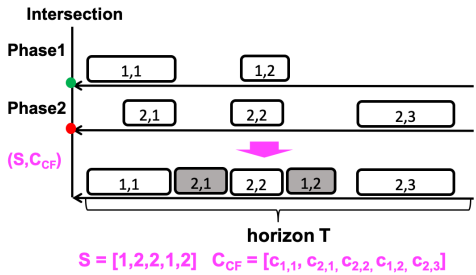}
\caption{The resulting control flow $(S, C_{CF})$ calculated by scheduling agents: each block represents a vehicular cluster. The shaded blocks represent the delayed clusters.}
\label{cf}
\end{figure}

\begin{algorithm}
\begin{algorithmic}[1]
\Require~~ 1) $(s, pd, t, d)$ of $S_{k-1}$ ; 2) $s_k$ 
\State $i = s_k$; $c =$ next job of phase $i$
\State $pst = t + MinSwitch(s, i)$ \Comment{Permitted start time of $c$}
\State $ast = \max(arr(c), pst)$  \Comment{Actual start time of $c$}
\If{$s \neq i$ and $pst > arr(c)$} 
$ast = ast + slt_i$ \EndIf 
\State $t = ast + dep(c) - arr(c)$ \Comment{Actual finish time of $c$}
\If{$s \neq i$} { $pd = t - pst $} \Else{ \quad$pd = pd + (t- pst)$} \EndIf 
\State $d = d + |c| \cdot (ast - arr(c))$ \Comment{Total accumulative delay}
\State\Return $(pd, t, d)$
\end{algorithmic}
\caption{Calculate $(pd, t, d)$ of $S_k$}
\label{pd}
\end{algorithm}


\section{Cooperative Algorithm}
In this section, we introduce an algorithm that enables scheduling agents to assign velocities to the approaching vehicles through V2I communication. The goal is to further reduce cumulative delay compared to the baseline case  above, where each vehicle passively follows timing plans at their free flow speeds. In brief, each scheduling agent computes a schedule every second. Based on the current cluster model and arrival prediction at any instant, the earliest time that any given vehicle can feasibly access the corresponding intersection (called the permitted starting time) can be determined by solving a single-machine scheduling problem to minimize cumulative delay. With these permitted times, the speeds of vehicles are adjusted by the scheduling agent through communication, and then the platoons are shifted to a new shape that incurs smaller cumulative delay. Basically, we revise the cluster representation through two operations: a) re-sequencing vehicles to minimize delay, and b) sending control message to adjust speed of vehicles. 

\subsection{Scheduling Information}
The schedule generated by the scheduling agent contains useful timing information including the permitted start time ($pst$), actual start time ($ast$), phase time ($pd$) and arrival time ($arr$), etc. The approaching vehicles are able to apply these information to change their movement in order to achieve better traffic flow that has smaller delay. The original schedule-driven traffic control uses this information to control traffic signal timing and thus control vehicles in an indirect way. However, as CAV technologies are incorporated, the timing information allows the intersection to control vehicles directly through V2I communication. For instance, with $pst(c_i)$, we know the earliest time to cross the current intersection for cluster $c_i$. Each vehicle could check their $arr(c_i)$ to decide whether to speed up or slow down under a pre-defined safety constraint. If $arr(c_i) > pst(c_i)$, the vehicles of $c_i$ could speed up to decrease the corresponding finish time $t$. On the other hand, if $arr(c_i) \leq pst(c_i)$, it is beneficial for the first several vehicles of $c_i$ to maintain a lower speed and merge to the platoons from behind.

\subsection{Optimizing Schedule via Changing Platoons}
After acquiring the $pst$ and $arr$ of each cluster from the produced schedule, the scheduling agent starts to query the approaching vehicles about their current speed for improving delay further. We propose a greedy algorithm to improve the schedule with the speed information by iterating through each cluster and computing a new velocity that either shortens the corresponding phase or increases its crossing speed. First, we check whether the current cluster is a beginning of a new phase in Algorithm \ref{algo}. If it is a new phase, we record the $pst_p = pst(c_i)$ for computing reduction $\delta$ of previous phase length, where $p$ is index of phase and $pst_p$ is the starting time of $p_{th}$ phase. We consider two cases to compute $\delta$: 
\begin{enumerate}
\item If it is not that all previous vehicles need to wait for green time, $\delta$ is equal to the difference between travel time and the updated travel time of the last vehicle in previous phase, where the computation of the updated travel time will be described below.

\item  If $pst_p - \delta \leq pst_{p-1}$ then the previous phase is too short, we need to move current phase forward and set $\delta  = 0$. After computing $\delta$, we know if the current phase is able to start earlier.
\end{enumerate}

\begin{algorithm}
\begin{algorithmic}[1]
\State Apply forward recursion process
\State Retrieve the schedule solution $(S^*,C^*_{CF})$.
\State Query all vehicles about their current velocity $v_i$.
\For {$i = 1$ to $C^*_{CF}$}
\State Get $pst(c_i)$ and $s_i$
\If{$s_i \neq phase$} \Comment{new phase start}
\State $pst_p = pst(c_i)$; $phase = s_i$; $\delta = 0$
\If{$\textit{end} > pst_{p-1}$}  \Comment{not all vehicles wait}
\State $\delta = \textit{end} - \max(pst_{p-1}, \textit{updated\_end})$ 
\EndIf
\If{$pst(c_i) - \delta \leq pst_{p-1}$}  
\State $\delta = 0$ \Comment{$\delta$ is not too large.}
\EndIf
\State $\textit{end} = \textit{updated\_end} = 0$;$p = p + 1$
\EndIf
\State $pst(c_i) = pst(c_i) - \delta$ \Comment{update permitted start time}
\State $\textit{end} = \max(\textit{end}, arr(c_i))$ 
\State Compute $v_i'$  and $\textit{updated\_end}$ by Algorithm \ref{expect_speed}.
\If{$\textit{is\_safe}(v_i')$}
\State Send speed control message to all vehicles in $c_i$
\EndIf
\EndFor
\end{algorithmic}
\caption{Splitting clusters according to previous schedule}
\label{algo}
\end{algorithm}

We maintain two variables, $\textit{end}$ and $\textit{updated\_end}$, to record the finish time of the previous phase before and after changing speed.  $\textit{end}$ is updated with $arr(c_i) $ in Line $14$, and $\textit{updated\_end}$ is updated based on the new speed computed in Line $15$ by Algorithm \ref{expect_speed}. The $\delta$ used for moving the phase finish time forward can be calculated from the difference between $\textit{end}$ and $\textit{updated\_end}$. Other than $\textit{updated\_end}$, the new speed in the control message that is sent is also computed by Algorithm \ref{expect_speed}.  Algorithm \ref{expect_speed} is based on $pst(c_i)$ and $arr(c_i)$, where $t_{c}$ in Line $1$ is  current time, Line $2$ defines a ratio $\gamma$ that is current travel time over the updated travel time, and the threshold in Line $3$ is used for focusing on those critical vehicles and avoiding unnecessary speed change to lower communication overhead. If deviation of $\gamma$ from $1$ is large enough, we will update the speed according to a speed planning function $\textit{new\_speed}(\cdot)$, where $\textit{new\_speed}(\cdot)$ is a function that assigns speed for each vehicle based on current velocity $v_i$ and $\gamma$. A suitable function based on the Intelligent Driver Model (IDM) \cite{treiber2000congested} is
\begin{equation}
v_i' =  v_i + a_{max}(1 - \gamma^{-\omega} - (\frac{s^*}{s})^2) \approx  v_i + a_{max}(1 - \gamma^{-\omega} )
\end{equation}
where $a_{max}$ is the maximum acceleration and $\omega$ is acceleration exponent. We assume the current headway $s^*$ is much larger than desired headway $s$, so that we can ignore the correction term. On the other hand, those vehicles with smaller deviation of $\gamma$ from $1$ are at the original speed for maintaining headway and space between vehicles. As we have the new speed,  $\textit{updated\_end}$ is updated in order to compute $\delta$ for modifying schedule. Figure~\ref{flow_chart} shows the three steps involved in changing a platoon sequentially and continually in a real-time manner.
   
 \begin{figure}[!htbp]
\centering
\includegraphics[scale = 0.45]{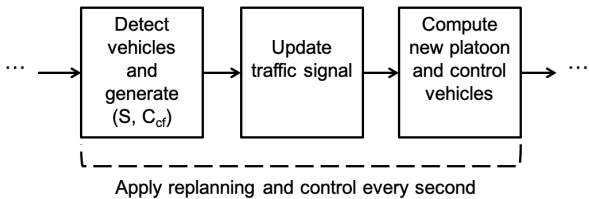}
\caption{The replanning and control cycle}
\label{flow_chart}
\end{figure}

\begin{algorithm}
\begin{algorithmic}[1]
\Require: $v_i$, $arr(c_i)$, $pst(c_i)$, $\textit{updated\_end}$
\State Get current time $t_c$ 
\State $\gamma = (arr(c_i)-t_c) /(pst(c_i) - t_c)$
\If{$\gamma > \textit{thr}_{up}$ and $\gamma < \textit{thr}_{down}$}
\State $v_i' = \textit{new\_speed}(v_i ,\gamma$)
\State $arr' = v_i/v_i' \times (arr(c_i) - t_c) + t_c$
\State $\textit{updated\_end} = \max(\textit{updated\_end}, arr')$
\Else
\State $\textit{updated\_end} = \max(\textit{updated\_end}, arr(c_i))$
\EndIf

\State \Return: $v_i',  \textit{updated\_end}$

\end{algorithmic}
\caption{Calculate($v_i',  \textit{updated\_end}$) of $c_i$ }
\label{expect_speed}
\end{algorithm}

Based on Algorithm \ref{algo}, we can either speed up or slow down vehicles to improve the phase schedule (i.e., timing plan).  As the vehicles speed up, it provides more space for the rear vehicles and makes the platoon more compact. For slowing down, it helps the approaching vehicles to keep a higher speed when crossing intersections. We prove that by applying Algorithm \ref{algo}, the cumulative delay is guaranteed to be less than the previous schedule. We state the following theorem of our proposed algorithm. 
\begin{theorem}
Let $S$ be the schedule produced by the baseline scehdule-driven traffic control approach. The cumulative delay of schedule $S'$, produced from schedule $S$ by applying Algorithm \ref{algo}, is less and equal to the cumulative delay of $S$
\label{opt}
\end{theorem}
\begin{proof}
We have two different situations. First, if the vehicle of $c_i$ is speeding up, we know $arr'$ is less than $arr(c_i)$. Although the delay contribution $|c_{i}| \cdot (ast(c_i) - arr') = |c_{i}| \cdot (arr' - arr') = 0$, we have smaller $ast(c_i) = \max(arr', pst(c_i)) = arr' < arr(c_i)$ and thus shorter phase duration $pd$. Second, if the vehicle of $c_i$ is slowing down, $arr'$ is larger than $arr(c_i)$ and $|c_{i}| \cdot (ast(c_i) - arr') \leq |c_{i}| \cdot (ast(c_i) - arr(c_i))$, where $ast(c_i) = pst(c_i)$. Thus, the cumulative delay is always less than the previous one.
\end{proof}

The cluster size $|c_i|$ affects the number of vehicles in each phase and thus has impact on how many vehicles will speed up to contribute to delay reduction in Algorithm \ref{algo}. In the following section, we will discuss how clustering affects the performance of the proposed algorithm.

\subsection{Cluster Size and Delay-Capacity Tradeoff}
As mentioned in the previous sections, vehicles entering an intersection are clustered together if they are traveling within a pre-specified interval of one another. By changing this pre-specified interval, the cluster distribution is changed. If we have a larger interval, we could easily have large cluster size under high traffic demand. Although it could be computationally advantageous as the scheduling agent is computing schedule for large clusters and has higher throughput (or capacity, i.e., it does not require signal phases to switch too frequently), it may lose some benefits of using the cooperative algorithm, which requires better partition of platoons. On the other hand, it would be computational-intensive if the cluster size is always set to one. However, since a smaller interval is used, a favorable partition of platoon is generated with minimum cumulative delay, and it benefits the cooperative algorithm through either splitting or merging clusters accurately. We can obtain a tradeoff between maximizing capacity and minimizing delay by both tuning the interval and applying the cooperative algorithm as well. The benefit of using the cooperative algorithm is different on the different intervals, but it is guaranteed that the delay of both large and small intervals is improved by Theorem~\ref{opt}. 

Another possible way to improve the performance of applying a larger interval further is that we only reschedule the clusters with larger delay and their adjacent clusters in a single-vehicle scheduling manner, which is taking each vehicle as a single cluster. First, we have to identify a cluster $c_i$ with the largest delay, which is $\Delta = (ast(c_i) - arr(c_i))$, and all the vehicles arriving during that interval as a single batch. Then, single-vehicle scheduling is applied, so that it is unnecessary to run single-vehicle scheduling on all vehicles within the current look-ahead horizon. Finally, we run the Algorithm \ref{algo} to change the speeds of vehicles within that horizon. To improve delay continually, we may identify a second cluster with larger delay and repeat above process. By incorporating this additional rescheduling, we can reshuffle vehicles in a batch to acquire smaller delay and more flexible computation compared to the original approach.

\section{Experimental Evaluation}

In this section, we evaluate the proposed cooperative algorithm through comparison to two different approaches: a baseline schedule-driven method and a fixed timing method in a connected vehicle environment with perfect information. The version of the schedule-driven method we apply here is Expressive Real-Time Intersection Scheduling (ERIS) \cite{goldstein2018expressive}, which maintains separate estimates of arrival traffic for each lane and enables it to more accurately estimate the effects of scheduling decisions than original schedule-driven approach. For accurate scheduling information, the proposed cooperative algorithm is also implemented based on ERIS. The fixed timing method is to match the given traffic demand according to Websters formula. We present a comparison over the proposed method, ERIS, and a fixed timing method on two separate networks.

The evaluations of the proposed method are ran on the Simulation of Urban MObility (SUMO), which is a microscopic traffic simulator that simulates continuous road traffic for large road networks and control of traffic signals and vehicles. To retrieve vehicle information and manipulate their behavior, we interface through Traffic Control Interface (TraCI) \cite{wegener2008traci}. We assume that each vehicle has its own route as it passes through the network and measure how long a vehicle must wait for its turn to pass through the intersections (the delay or time loss). Tested traffic volume is averaged over sources at network boundaries. To assess the performance boost provided by the cooperative algorithm, we measure the average waiting time of all vehicles over five runs. All simulations run for $1$ hour of simulated time. Results for a given experiment are averaged across all simulation runs with different random seeds. In order to eliminate the effects of simulation start up and termination, the time loss of vehicles arriving within the middle $40$ minutes is reported. Vehicle arrivals are modeled as a Poisson process where the average arrival rate is set according to the desired level of congestion. We choose the IDM model with $\omega = 4$ for computing new speed, and other parameters are listed below: a) speed limit $65km/h$. b) $\textit{thr} = 1\pm0.4$. c) $a_{max} = 5m/s$.

\subsection{Simulation Results}
\begin{table*}[tp]
\centering
\scalebox{0.9}{
   \begin{tabular}{*{11}{c}}
   \toprule
      \multirow{2}{*}{}& \multicolumn{10}{c}{ (a) Average Delay (second) of Single Intersection}  \\
    
   \cmidrule(l){2-11}    & \multicolumn{2}{c}{Schedule-driven} \textbf{(0s)} &  \multicolumn{2}{c}{Cooperative}\textbf{(0s)} &  \multicolumn{2}{c}{Schedule-driven}\textbf{(3s)}&  \multicolumn{2}{c}{Cooperative} \textbf{(3s)}&   \multicolumn{2}{c}{Fixed timing}   \\
   & mean  & std. & mean & std.  & mean  & std. & mean & std  & mean &std. \\
    \midrule
 High demand& 33.52 &26.86 &  27.23& 21.86&37.52 &37.84 & 29.67& 26.83& 40.62  & 28.87\\
   Medium demand & 19.32& 20.10&  17.75 &15.41&19.62& 17.07& 17.83 &15.67& 25.67 & 18.64\\
   Low demand& 12.02 &13.97&  11.58 & 13.21 &12.69 &14.59& 12.38 & 13.73  & 22.44 & 19.00 \\
  \end{tabular}
  }

  \centering
 \scalebox{0.9}{
   \begin{tabular}{*{11}{c}}
   \toprule
      \multirow{2}{*}{}& \multicolumn{10}{c}{ (b) Average Delay (second) of Three Intersections}  \\
    
   \cmidrule(l){2-11}    & \multicolumn{2}{c}{Schedule-driven}\textbf{(0s)} &  \multicolumn{2}{c}{Cooperative}\textbf{(0s)} &  \multicolumn{2}{c}{Schedule-driven}\textbf{(3s)} &  \multicolumn{2}{c}{Cooperative}\textbf{(3s)} &   \multicolumn{2}{c}{Fixed timing}   \\
   & mean  & std. & mean & std.  & mean  & std. & mean & std  & mean &std. \\
    \midrule
  High demand& 33.60 &22.33 & 29.75& 19.83 & 31.48 &21.68 & 29.55& 20.55 &  42.76 & 26.78\\
   Medium demand & 25.98&18.69 & 25.40 &17.19 & 25.21&19.23 & 24.22 &17.30& 32.27 & 21.96\\
   Low demand& 17.82 &15.93& 18.10 &15.27  &12.69 &14.59& 12.38 & 13.73 & 28.19 & 19.71 \\
    \bottomrule
  \end{tabular}
 }

   \caption{Average delay of single intersection and three intersections with different clustering intervals.}
  \label{tb1}
\end{table*}

%
%

First, a network consisting of a single intersection with two lanes on the main road and one lane on the side street is examined. To explore how the cooperative algorithm performs under different demand, we categorize traffic demand into three different groups: low (363 cars/hour), medium (750 cars/hour), and high (1250 cars/hour). Other than different traffic demand, we test how different clustering intervals affect the performance. Average delay per vehicle (in seconds) and standard errors across a range of vehicle volumes are presented in Table \ref{tb1}(a). We also present the improvement when measured against fixed timing at each level of congestion.

With single-vehicle scheduling, where we have a smaller clustering threshold and take each vehicle as a single cluster, the cooperative algorithm outperforms schedule-driven approach and fixed timing plans for all tested levels of congestion. Under high traffic demand setting, the improvement of the cooperative algorithm is the largest and up to $19\%$ and $33\%$ compared to the two other approaches, which means more vehicles under this demand can be gained from the speed adjustment and scheduling. On the other hand, the performance is comparable for all three approaches under low and medium demand. The cooperative algorithm can only improve the delay of a small portion of vehicles, and most improvement in delay comes from obeying the produced schedule.

As the clustering threshold increases to $3$ seconds, the improvement of the cooperative algorithm over schedule-driven traffic control becomes more evident and ranges up to $21\%$ under high traffic demand. One reason is that for schedule-driven traffic control the reshuffling effect caused by scheduling is not able to create a platoon with smaller delay ($37$s v.s. $33$s) because of large average cluster. However, the delay of the cooperative algorithm only increases from $27$s to $29$s. If sufficient computation power is available we would only suggest to use smaller clustering threshold in order to have better performance of the cooperative algorithm.

Other than single intersection, a model of three connected intersections is evaluated in Table \ref{tb1}(b). The delay of the cooperative algorithm is still less than the schedule-driven traffic control approach under high traffic demand with two different clustering intervals as predicted in Theorem \ref{opt}. However, the improvement is not as large as compared with single intersection. When we have three connected intersections, we observe that the delay with larger clustering interval  is less than the case of single-vehicle scheduling as shown in the Table \ref{tb1}(b). Basically, a smaller clustering interval provides a better optimality for single intersection, but it may cast more traffic demand for neighbor intersections and harm network-level performance (or global performance). A better network coordination should resolve this issue \cite{hu2018bi}. The delay of the cooperative algorithm in the second and forth column seems not be affected by the clustering interval in a network a lot. The shift of platoons somehow provides a robustness when applying traffic scheduling within a transportation network.

\subsection{Partial Penetration of CAV}
Understanding the performance under different penetration rates is an essential issue for deploying the proposed methods under realistic scenarios. In the results presented thus far, we have assumed that all vehicles are controllable by the proposed algorithm. However, it may be difficult in practice to achieve such a high penetration rate. 

Table~\ref{penetration} presents the delay of the proposed algorithm under different penetration rates. We assume that partial vehicles are able to control their speed according to the messages sent by the intersection. Three penetration rates are tested to compare with full penetration rate. The results show that the proposed algorithm can still provide a considerable $14\%$  improvement compared with the baseline case as the penetration rate is $50\%$ and $70\%$. The improvement drops to $8\%$ at a penetration rate of $30\%$ which still represents a reasonable reduction in overall congestion. 

\begin{table}[!htbp]
\centering
  \scalebox{0.8}{
  \begin{tabular}{*{9}{c}}
   \toprule
      \multirow{2}{*}{}& \multicolumn{8}{c}{ Average Delay (second)}  \\
    
   \cmidrule(l){2-9}    & \multicolumn{2}{c}{100\%} &  \multicolumn{2}{c}{70\%} &  \multicolumn{2}{c}{50\%} &  \multicolumn{2}{c}{30\%}   \\
   & mean  & std. & mean & std.   & mean &std.  & mean &std. \\
    \midrule
 High demand& 27.23 &21.86 & 29.22& 23.64 &29.57   & 24.31&31.58   & 26.89\\
   Medium demand & 17.75&15.41 & 19.12 &17.50& 19.29  &16.38 &  19.40 &17.28\\
   Low demand& 11.58 & 13.21& 11.55 &13.85  & 11.55 &13.90 &  11.87 & 13.56\\
    \bottomrule
  \end{tabular}
  }
   \caption{Average delay under different penetration rates.}
     \label{penetration}
\end{table}

\section{Conclusion}
In this paper, we described a cooperative framework designed to enable the interaction between CAV and schedule-driven traffic control. The scheduling algorithm generates useful scheduling information continually to reflect real traffic conditions by a strategy of frequent replanning. This information is used to guide vehicles on how to adjust their velocity for minimizing delay. The cooperative algorithm is specified as a way to shift approaching platoons through both reshuffling vehicles that are arriving from different, competing directions and changing their speed. The proposed approach demonstrates that a lower delay can be obtained. Moreover, we can still get reasonable performance under lower penetration rate of CAV.

 The cooperative system was evaluated on a simulation model of a traffic signal control problem. Results showed that the interaction between scheduling and CAV improves average delay overall in comparison to both the baseline schedule-driven traffic control approach and a fixed timing approach, and that solutions provide substantial gain in highly congested scenarios. Future work will focus on the design of how to achieve a network-wide coordination through this vehicle-to-infrastructure communication.

\bibliographystyle{IEEEtran}
\bibliography{split}

\end{document}